\documentclass{article}
\usepackage[a4paper, total={6in, 8in}]{geometry}
\usepackage{natbib}

\usepackage{tikz}
\usepackage{amsmath}
\usepackage{amsthm}
\usepackage{hyperref, cleveref}
\hypersetup{
    colorlinks=true,
    linkcolor=blue,
    filecolor=blue,      
    urlcolor=blue,
    citecolor=blue,
    pdfpagemode=FullScreen,
    }
\usepackage{amsfonts}

\usetikzlibrary{shapes.geometric, positioning}

\theoremstyle{plain}
\newtheorem{theorem}{Theorem}[section]
\newtheorem{proposition}[theorem]{Proposition}
\newtheorem{lemma}[theorem]{Lemma}

\theoremstyle{definition}

\theoremstyle{remark}

\newcommand{\ba}{\mathbf{a}}
\newcommand{\bad}{\mathrm{BAD}}
\newcommand{\well}{\widehat{\ell}}

\newcommand{\bleftm}{\bigl\{\!\!\bigl\{}
\newcommand{\brightm}{\bigl\}\!\!\bigl\}}

\title{On dimensionality of feature vectors in MPNNs}
\usepackage{times}

\author{C\'{e}sar Bravo$^1$ \and Alexander Kozachinskiy$^{23}$ \and Crist\'{o}bal Rojas$^{12}$}
\date{%
    $^1$Instituto de Ingeniería Matemática y Computacional, Universidad Católica de Chile\\%
    $^2$Centro Nacional de Inteligencia Artificial, Chile\\
    $^3$Instituto Milenio Fundamentos de los Datos, Chile
}

\begin{document}
\maketitle
\begin{abstract}
We revisit the classical result of Morris et al.~(AAAI'19) that message-passing graphs neural networks (MPNNs) are equal in their distinguishing power to the Weisfeiler--Leman (WL) isomorphism test. 

Morris et al.~show their simulation result with ReLU activation function and $O(n)$-dimensional feature vectors, where $n$ is the number of nodes of the graph. By introducing randomness into the architecture, Aamand et al.~(NeurIPS'22) were able to improve this bound to $O(\log n)$-dimensional feature vectors, again for ReLU activation, although at the expense of guaranteeing perfect simulation only with high probability. 

Recently, Amir et al.~(NeurIPS'23) have shown that for any non-polynomial analytic activation function, it is enough to use just 1-dimensional feature vectors. In this paper, we give a simple proof of the result of Amit et al.~and provide an independent experimental validation of it. 


\end{abstract}

\section{Introduction}
A plethora of real-life data is represented as a graph. A common deep-learning architecture for this kind of data is a graph neural network (GNN)~\cite{scarselli2008graph}. In this paper, we focus on message-passing graph neural networks (MPNN), which are rather simple architectures that have proved to be quite useful in practice.

MPNNs work as follows. First, as inputs to them, we consider simple undirected graphs with node labels. MPNNs work in layers, where each node of a graph has its own \emph{feature vector} that is updated at each layer. The initial feature vector of a node is some encoding of the label of this node (typically a one-hot encoding if the number of different labels is not too big). In each layer, the feature vector of a node $v$ is updated as follows. One takes the sum of feature vectors of the neighbors of $v$, multiplies it by some matrix (with learnable entries), and adds the feature vector of $v$ itself, multiplied by some other matrix (with learnable entries as well). After that, a non-linear activation function is applied coordinate-wise to the resulting vector, yielding the new feature vector of $v$.

Why do MPNNs perform well in practice?  To answer this question, one needs to understand which specific architectural characteristics of MPNNs can provide guarantees on their performance, for example in terms of their trainability, generalization power, or their ability to fit the data. In this paper we will focus on the ability of MPNNs to accurately represent the data, a property commonly referred to as \emph{expressive power}.

 A main feature of MPNNs is that, by design, they compute functions on graphs that assign the same output to isomorphic graphs. That is, the functions they compute are always invariant under isomorphisms. Which invariant functions can MPNNs compute?  It turns out that this question is related to the MPNNs's ability to distinguish non-isomorphic graphs, in the sense of being able to produce different outputs for graphs that are not isomorphic. An MPNN architecture is said to be \emph{complete} for a class of graphs if it can distinguish all pairs of non-isomorphic graphs within the class. As shown in~\cite{NEURIPS2019_71ee911d}, a MPNN architecture is universal (can approximate any invariant function arbitrarily well) within a class of graphs, if and only if it is complete for that class. Therefore, it becomes relevant to identify the characteristics of MPNN architectures that guarantee maximal distinguishing power.

It was observed independently by~\cite{DBLP:conf/iclr/XuHLJ19} and~\cite{morris2019weisfeiler}  that the distinguishing power of MPNNs is upper bounded by the so-called Weisfeiler-Leman (WL) test~\cite{leman1968reduction}, in the sense that MPNNs are incapable of distinguishing two node-labeled graphs that are not distinguished by the WL-test. Let us briefly explain how the test works. Given two node-labeled graphs, the test iteratively updates the labels and checks, after each update, whether every label appears equally many times in both graphs. If a discrepancy arises during this process, the graphs are reported as non-isomorphic. In every update, the label of a node is replaced by a pair consisting of its current label together with the multiset of labels of its neighbors. It is known that there are non-isomorphic graphs that cannot be distinguished by this test, meaning that the test will consistently fail to report them as non-isomorphic. This implies that MPNNs are unable to distinguish all pairs of non-isomorphic graphs.


 Do architectures that are used in practice attain this WL upper bound? In particular, what characteristics, such as the type of activation function or the dimension of feature vectors, are required to guarantee maximal expressive power? 
 
Morris et al.~answered this question by constructing, for any given graph, an MPNN that can perfectly simulate the WL-test on it. The characteristics of their MPNN are, however, graph-dependent: ReLU activation function works, and the dimension of feature vectors has to be $O(n)$, where $n$ is the number of nodes in the graph. According to their result, if one is working with graphs with around $1000$ nodes, feature vectors of dimension $d=1000$ should be used. This is in stark contrast to what is done in practice, where the dimension of feature vectors is rather small (typically a few hundreds) and independent of graph size. Furthermore, the parameter values in their construction also depend on the input graph. 

An exponential improvement to the construction of Morris et al.~was developed in~\cite{aamand2022exponentially}, where the dimension of the feature vectors was reduced to $O(\log(n))$. Their construction also has the advantage of being partially uniform, allowing a single MPNN to be used on all graphs of a certain size. Their architecture is, however, randomized, and perfect simulation is only guaranteed with high probability.


Does the dimension of feature vectors necessarily have to grow with the size of the graphs?  What guarantees can be provided for MPNNs with continuous feature vectors of constant dimension, such as those used in practice?

A surprising answer was given recently by~\cite{amir2023neural}. They showed that for \emph{any} non-polynomial analytic activation function (like the sigmoid), MPNNs with one-dimensional feature vectors are already as powerful as the WL-test. In particular, the dimension of the feature vectors does not have to grow with the size of a graph to guarantee full expressive power.

Amir et al.~obtain this as a corollary of a rather technical result they call the \emph{finite witness theorem}. In this paper, we give a direct and extremely simple proof that MPNNs with one-dimensional feature vectors and a single parameter are uniformly equivalent to the WL test, independently of the size of the graphs.

\paragraph{Key contributions.} 
\begin{enumerate}
    \item Our main theoretical contribution is a simple proof of the following: 
\begin{theorem}
\label{thm_main}
    Let $\mathbf{a}\colon \mathbb{R}\to\mathbb{R}$ be any analytic non-polynomial function. Then two node-labeled graphs can be distinguished by the WL test if and only if they can be distinguished by some MPNN with 1-dimensional feature vectors and activation function $\mathbf{a}$.
\end{theorem}

Our architecture has a single parameter $\gamma$, unlike the architecture of~\cite{amir2023neural} that has two (one is the bias). Our proof actually shows that, for all choices of $\gamma\in(0,1)$ (except for possibly countably many), the corresponding one-dimensional MPNN is \emph{fully uniform} in the sense that it is equivalent to the WL-test on any input graph of any size, provided that the operations with real numbers are performed with sufficiently many bits of precision.

\item  We experimentally validate this theoretical result by demonstrating that for nearly every random $\gamma \in (0,1)$, the simplest possible one-dimensional MPNN is capable of achieving perfect simulation of the WL algorithm simultaneously for all graphs in two different collections of 300 random graphs each. We also explored how the minimum number of precision bits required to guarantee perfect simulation depends on the size of the graphs. Our results suggest that logarithmically many bits of precision are enough. 
\end{enumerate}

\paragraph*{Our technique.} To show equivalence between MPNNs and the WL-test, one has to construct, for any given graph, some MPNN with the property that after any number of iterations, two nodes receive the same feature vectors if and only if they receive the same WL labels. The proof goes by induction on the number of iterations. In an MPNN,  the multiset of labels of a node is encoded by the sum of the current feature vectors of the neighbors. To ensure the injectivity of this encoding, the construction of Morris et al.~maintains \emph{linear independence} of different feature vectors after each iteration. However, since the number of different labels in principle can be as large as $n$ (the number of nodes), this argument requires the dimension of feature vectors to be at least $n$, as the number of linearly independent vectors in $\mathbb{R}^d$ can not be larger than $d$.

To overcome this difficulty, we just observe that for the argument it is enough to require linear independence over \emph{rational numbers}. That is because multisets of labels are encoded as linear combinations of feature vectors, with coefficients that correspond to the multiplicities of the labels in these multisets, and multiplicities are always integer numbers. Now, the point is that even in $\mathbb{R}$, we can pick arbitrarily many numbers that are linearly independent over $\mathbb{Q}$, for instance, square roots of prime numbers, $\sqrt{2}, \sqrt{3}, \sqrt{5}$, and so on. Our main technical contribution is a demonstration that any non-polynomial analytic activation function can map arbitrarily many different numbers into a collection of numbers that are linearly independent over $\mathbb{Q}$, under a suitable choice of a single parameter.

\paragraph*{Extensions and related work.} 

The connection between GNNs and the WL-test has inspired the development of new, more expressive GNN architectures, mimicking \emph{higher-order} versions of WL test~\cite{cai1992optimal} and enjoying the same theoretical equivalence to them~\cite{morris2019weisfeiler,maron2019provably}. This connection was also established for more general models of graphs, such as geometric  graphs~\cite{DBLP:conf/icml/JoshiBMCL23} and relational graphs~\cite{DBLP:conf/log/Barcelo00O22}. In some cases, the new, WL-inspired architectures, demonstrated an improvement over standard MPNNs as well as over state-of-the-art models, for example, in the case of learning 3-dimensional point clouds~\cite{li2023distance}. We point out the interested reader to a survey on the use of the WL-test in machine learning~\cite{JMLR:v24:22-0240}.

We briefly discuss, how our main theoretical result extends to higher-order MPNNs, mimicking the so-called \emph{non-folklore} higher-order versions of the WL test, see~\cite{huang2021short}. More precisely, we prove that one-dimensional feature vectors and any non-polynomial analytic activation function are still enough for the equivalence between $k$-order non-folklore WL-test and $k$-order MPNN, for any $k \ge 2$. We note, however, that only for $k = 3$, the $k$-order non-folklore WL-test surpasses the classical WL-test in its expressive power. The corresponding 3-order MPNNs require $O(n^3)$ space and $O(n^4)$ time per iteration, which considerably limits their practical applications. It would be interesting to obtain an effective MPNN architecture, mimicking the \emph{folklore} 2-order WL-test, which has the same expressive power as the non-folklore 3-order WL test but works in $O(n^2)$ space and $O(n^3)$ time per iteration. We hope that our technique can be useful in bounding the dimension of feature vectors. It should be noted, however, that~\cite{maron2019provably} established an $O(n^2)$-time and $O(n^3)$-space architecture, having at least the same expressive power as the folklore 2-order WL test, although this architecture is not a direct ``sample'' from this test.

\section{Preliminaries}

 A multiset is a function $f$ from some set $S$ to $\mathbb{N}$ (elements of $S$ are taken with some finite \emph{multiplicities}, specified by $f$). We use brackets $\bleftm, \brightm$ to denote multisets. More precisely, if $A, B$ are sets (and $A$ is finite) and $\phi\colon A \to B$ is a function, we write $\bleftm \phi(a)\mid a\in A \brightm$ for the multiset over elements of $B$, where each element $b\in B$ is taken with the multiplicity $|\phi^{-1}(b)|$ (that is, how many times $b$ appears as the value of $\phi$ when we go through elements of $A$).
 
We consider simple undirected graphs with node labels taken from some set $L$. Such a graph can be given as a triple $G = \langle V, E, \ell\rangle$, where $V$ is the set of nodes of the graph, $E \subseteq \binom{V}{2}$ is the set of edges of the graph (some set of 2-element subsets of $V$), and $\ell\colon V\to L$ is the node-labeling function. For $v\in V$, we denote by $N_G(v)$ the set of neighbors of $v$ in the graph $G$, that is, $N_G(v) = \{u\in V\mid \{u,v\}\in E\}$. If $G$ is clear from the context, we drop the subscript $G$.

\paragraph*{Weisfieler Leman test.}
The Weisfeiler--Leman algorithm receives on input a node-labeled graph $G = \langle V, E, \ell\rangle$ and produces a sequence of ``node labeling'' functions \[\phi_G^{(0)}, \phi_G^{(1)}, \phi_G^{(2)}\colon V\to ..., \]
defined iteratively as follows:
\begin{itemize}
    \item $\phi_G^{(0)}(v) = \ell(v)$ for $v\in V$.
    \item $\phi_G^{(t+1)}(v) = \left(\phi_G^{(t)}(v), \bleftm\phi_G^{(t)}(u) \mid u\in N(v)\brightm\right)$
\end{itemize}

The Weisfieler Leman test receives on input two labeled graphs $G_1 = \langle V_1, E_1, \ell_1\rangle$ and $G_2 = \langle V_2, E_2, \ell_2\rangle$. It runs the Weisfieler Leman algorithm on both of them, obtaining two sequences of node labeling functions $\{\phi^{(t)}_{G_1}\}_{t\ge 0}$ and $\{\phi^{(t)}_{G_2}\}_{t\ge 0}$. If for some $t\ge 0$ the test finds out that the multisets of WL-labels in $G_1$ and $G_2$ after $t$ iterations are different:
\[\bleftm \phi^{(t)}_{G_1}(v)\mid v\in V_1  \brightm \neq \bleftm \phi^{(t)}_{G_2}(v)\mid v\in V_2  \brightm,\]
the test outputs that these graphs can be distinguished. It is standard that whether or not the WL test distinguishes two graphs can be checked in polynomial time.

\paragraph{MPNNs.} A  $T$-layer \emph{message-passing graph neural network} (MPNN) $\mathcal{N}$ with $d$-dimensional feature vectors and activation function $a\colon \mathbb{R}\to\mathbb{R}$ is specified by:
\begin{itemize}
    \item a label-encoding function $\mathbf{e}\colon L\to\mathbb{R}^d$;
    \item a sequence of $t$ pairs of $d\times d$ real matrices \[W_1^{(1)}, W_2^{(1)}, \ldots, W_1^{(T)}, W_2^{(T)} \]
    (in practice, elements of these matrices are to be learned).
\end{itemize}
Given a node-labeled graph $G = \langle V, E, \ell\rangle$, the MPNN $\mathcal{N}$ works on it as follows. First, it computes initial feature vectors using the label-encoding function:
\[f^{(0)}(v) = \mathbf{e}(\ell(v))\in\mathbb{R}^d, \qquad v\in V. \]
Then for $t = 0, \ldots, T-1$, it does the following updates of the feature vectors:
\begin{align}
\label{gnn_update}
    f^{(t+1)}(v) &= \mathbf{a}\left(W_1^{(t+1)} f^{(t)}(v) + W_2^{(t+1)}\sum\limits_{u\in N(v)} f^{(t)}(u)\right),\\
    f^{(t+1)}(v)&\in\mathbb{R}^d, v\in V.
\end{align}
(we assume that $\mathbf{a}$ is applied component-wise).

The output of $\mathcal{N}$ on the graph $G$ is defined as the sum of all feature vectors after $T$ iterations:
\[\mathcal{N}(G) = \sum\limits_{v\in V} f^{(T)}(v).\]
We say that two graphs $G_1, G_2$ are distinguished by $\mathcal{N}$ if $\mathcal{N}(G_1) \neq \mathcal{N}(G_2)$.

\paragraph{Non-polynomial analytic activation functions.}
We recall that an analytic function $f:\mathbb{R}\to \mathbb{R}$ is an infinitely differentiable function such that for every $x_0\in\mathbb{R}$, its Taylor series at $x_0$: 
\[
T(x)=\sum_{n=0}^\infty \frac{f^{(n)}(x_0)}{n\!}(x-x_0)^n 
\]
converges to $f(x)$ in a neighborhood of $x_0$. 

We will require our activation functions to be analytic and non-polynomial. Common examples are the sigmoid functions such as the Arctangent function, or the logistic function: 
\[
\sigma(x)=\frac{e^x}{1+e^x}.
\]

We will need the following well-known fact. For the reader's convenience, we also include a proof. 

\begin{proposition} 
\label{der}
Let $f: \mathbb{R}\to\mathbb{R}$ be an analytic function. Suppose there is a point $x_0$ where all the derivatives of $f$ vanish. Then $f$ is the constant zero function. 
\end{proposition}
\begin{proof}
We consider the set \[D=\{x\in\mathbb{R}:f^{(n)}(x)=0 \text{ for all }n\geq 0\}\]
and show that $D=\mathbb{R}$. First, we note that $x_0\in D$, so $D$ is nonempty. Second, we note that since $f$ is analytic, for any $x\in D$ there is a neighborhood of $x$ where $f$ equals its Taylor series at $x$. It follows that every point in this neighborhood is also in $D$, and thus $D$ is open. Finally, since all the $f^{(n)}$ are continuous, the sets $\{x:f^{(n)}(x)=0\}$ are all closed. Thus, being a countable intersection of closed sets, $D$ is also closed. Being open, closed, and nonempty, it follows that $D=\mathbb{R}$. 
\end{proof}

\section{Proof of the main result}\label{sec:proof}
In this section, we prove our main result.
\begin{theorem}[Restatement of Theorem \ref{thm_main}]
\label{thm_main2}
Let $\mathbf{a}\colon \mathbb{R}\to\mathbb{R}$ be any analytic non-polynomial function. Then
    two node-labeled graphs can be distinguished by the WL test if and only if they can be distinguished by some MPNN with 1-dimensional feature vectors and activation function $\mathbf{a}$.

\end{theorem}

One direction of that theorem is well-known and standard -- if two graphs can be distinguished by some MPNN, then they can be distinguished by the WL-test. In the rest of this section, we establish the other direction. 

We take two node-labeled graphs $G_1 = \langle V_1, E_1, \ell_1\rangle$ and $G_2 =\langle V_2, E_2, \ell_2\rangle$ that are distinguished by the WL test after $T$ iterations. We establish a $T$-layer MPNN $\mathcal{N}$ with 1-dimensional feature vectors and activation function $\mathbf{a}$ that distinguishes $G_1$ and $G_2$.

We need the following lemma.

\begin{lemma}
\label{invariant}
Let $\mathbf{a}\colon\mathbb{R}\to \mathbb{R}$ be any analytic non-polynomial function. Then
    for any labeled graph $G =\langle V, E, \ell\rangle$ and for any $T\ge 0$ there exists an MPNN with   1-dimensional feature vectors and activation function $\mathbf{a}$ such that for any $0\le t \le T$, the following conditions hold:
    \begin{itemize}
        \item $f^{(t)}(u) = f^{(t)}(v) \iff \phi^{(t)}(u) = \phi^{(t)}(v)$ for every $u, v\in V$. Here $f^{(t)}(v)$ denote the feature vector of the node $v$ in our MPNN after $t$ iterations (in our case, this is just a single real number), and $\phi^{(t)}(v)$ is the WL-label of the node $v$ after $t$ iterations.
        \item the set of real numbers $F_t = \{f^{(t)}(v)\mid v\in V\}$ is linearly independent over $\mathbb{Q}$. This means that for any distinct $x_1, \ldots, x_m \in F_t$ there exists no $(\lambda_1, \ldots, \lambda_m)\in\mathbb{Q}^m\setminus \{(0, 0, \ldots, 0)\}$ such that 
\[\lambda_1 x_1 + \ldots + \lambda_m x_m = 0.\] 
    \end{itemize}
\end{lemma}
Let us explain, using Lemma \ref{invariant}, how to construct an MPNN $\mathcal{N}$ distinguishing $G_1$ and $G_2$. We apply the lemma to the graph $G$ which consists of a copy of  $G_1$ and a copy of $G_2$, with no edges between the copies and no shared nodes (sometimes it is called the disjoint union of two graphs). 
Observe that running the WL test on $G$ will assign the same WL labels to the nodes in $G_1$ and $G_2$ as if we run the test on them separately. Now, from the lemma we get the existence of some MPNN $\mathcal{N}$, satisfying its conditions for the graph $G$. Note that $\mathcal{N}$ can be run on $G_1$ and $G_2$ separately, assigning the same feature vectors to their nodes as if we run $\mathcal{N}$ on $G$. 

Let $F_T = \{f^{(T)}_{G_1}(v)\mid v\in V_1\}\cup \{f^{(T)}_{G_2}(v)\mid v\in V_2\}$ be the set of feature vectors that $\mathcal{N}$ assigns to $G$ after $T$ iterations. Notice that it is also a union of the sets of its feature vectors on $G_1$ and $G_2$. 
Since $G_1$ and $G_2$ are distinguished by the WL test after $T$ iterations, the multisets of WL labels of their nodes after $T$ iterations differ. This means that 
if we look at two ``formal'' sums of their WL labels
\[\sum\limits_{v\in V_1}\phi^{(T)}_{G_1}(v), \qquad \sum\limits_{v\in V_2}\phi^{(T)}_{G_2}(v),\]
there will be a label that either appears in one sum and not the other, or it appears in both but a different number of times. Due to the first condition of the lemma, feature vectors of $G$ after $T$ iterations are in a one-to-one correspondence with the WL labels. This means that the difference:
\[\sum\limits_{v\in V_1}f^{(T)}_{G_1}(v) - \sum\limits_{v\in V_2}f^{(T)}_{G_2}(v)\]
gives a non-trivial linear combination with rational (in fact, integer) coefficients of some numbers in $F_T$. Since the set $F_T\subseteq \mathbb{R}$ is linearly independent over $\mathbb{Q}$, this means that this difference is non-zero, giving us:
\[\mathcal{N}(G_1) =\sum\limits_{v\in V_1}f^{(T)}_{G_1}(v)\neq  \sum\limits_{v\in V_2}f^{(T)}_{G_2}(v) = \mathcal{N}(G_2),\]
as required.

Let us know establish Lemma \ref{invariant} by induction. Conditions of the lemma for $t = 0$ can be achieved via a label-encoding function that maps different labels into square roots of different prime numbers. The first condition is trivially satisfied due to the injectivity of this mapping, and the second condition holds because it is well-known that square roots of prime numbers:
\[\sqrt{2}, \sqrt{3}, \sqrt{5}, \ldots\]
are linearly independent over $\mathbb{Q}$.

Now, assume that we already constructed a $t$-layer MPNN, satisfying conditions of the lemma. We add a new layer to this MPNN so that the condition also holds after $t+1$ iterations of the WL test.

The new layer will be as follows:
\[f^{(t+1)}(v) = \mathbf{a}\left(\gamma n\cdot f^{(t)}(v) + \gamma \cdot\sum\limits_{u\in N(v)}  f^{(t)}(u)\right).\]
That is,  in our case the matrices $W^{(t+1)}_1, W^{(t+1)}_2$ from \eqref{gnn_update}  are $1\times1$ and are defined as \[W^{(t+1)}_1 = \gamma n, \qquad W^{(t+1)}_2 =\gamma,\] where $n$ is the number of nodes of $G$ and $\gamma\in [0, 1]$ is to be defined later.

We can rewrite this as:
\begin{align*}
    f^{(t+1)}(v) &= \mathbf{a}(\gamma x_v),\\
    x_v &= n\cdot f^{(t)}(v) + \sum\limits_{u\in N(v)}f^{(t)}(u).
\end{align*}
Using the the condition of our lemma for $t$, we have to establish that $x_u = x_v$ if and only if $\phi^{(t+1)}(u) = \phi^{(t+1)}(v)$ for $u, v\in V$. In other words, the value of $x_v$ uniquely determines the value of $\phi^{(t+1)}(v)$, and vice versa. It is easy to see that 
\[\phi^{(t+1)}(v) = \left(\phi^{(t)}(v), \bleftm\phi^{(t)}(u) \mid u\in N(v)\brightm\right)\]
uniquely determines the value of $x_v$, because $x_v$ is a function of $f^{(t)}(v)$ and $\sum_{u\in N(v)} f^{(t)}(u)$, which in turn are functions of $\phi^{(t)}(v)$ and $\bleftm \phi^{(t)}(u)\mid u\in N(v)\brightm$, correspondingly. 

We now establish that
\begin{equation}
    \label{x_v}
    x_v = n\cdot f^{(t)}(v) + \sum\limits_{u\in N(v)}f^{(t)}(u)
\end{equation}
uniquely determines $\phi^{(t+1)}(v)$. For that, we use the fact that the set of feature vectors $F_t$ after $t$ iterations is linearly independent over $\mathbb{Q}$. This means that by knowing the value of some linear combination of numbers from $F_t$ with rational coefficients, we can uniquely determine the coefficients of this linear combination (if there were two ways to obtain the same sum with different coefficients, by taking the difference of these two linear combinations, we would have obtained a non-trivial combination of numbers from $F_t$ that sums up to $0$). In particular, knowing the value of \eqref{x_v}, we first can determine the value $f^{(t)}(v)$ because it is the only number whose coefficient is at least $n$, the number of nodes. By subtracting $n$ from the coefficient before $f^{(t)}(v)$, we get the coefficients for the sum $\sum_{u\in N(v)}f^{(t)}(u)$, and these coefficients allow us to restore the multiset $\bleftm f^{(t)}(u)\mid u\in N(v)\brightm$. Using the condition that after $t$ iterations, WL labels and feature vectors are in a one-to-one correspondence, we can then uniquely restore $\phi^{(t)}(v)$ and $\bleftm \phi^{(t)}(u)\mid u\in N(v)\brightm$, and  thus $\phi^{(t+1)}(v)$, as required.

At this point, we have numbers $x_v, v\in V$ that satisfy the first condition of the lemma. However, the set $\{x_v\mid v\in V\}$ might not be linearly independent over $\mathbb{Q}$. Luckily, in our MPNN, to obtain features of nodes after $t+1$ layers, we then multiply these numbers by $\gamma$ and apply the activation function:
\[x_v \mapsto \mathbf{a}(\gamma x_v) =  f^{(t+1)}(v).\]
It remains to establish that there exists a choice of $\gamma$ such that,  if $\{x_v\mid v\in V\} = \{x_1, \ldots, x_m\}$ and $x_1, \ldots, x_m\in\mathbb{R}$ are distinct, then the numbers:
\[\ba(\gamma x_1), \ldots, \ba(\gamma x_m)\]
are linearly independent over $\mathbb{Q}$ (in particular, these numbers will have to be distinct, which will preserve the one-to-one correspondence with the WL labels). We derive this from the following lemma.

\begin{lemma}
\label{ind_lemma}
Let $\mathbf{a}\colon\mathbb{R}\to \mathbb{R}$ be any non-polynomial analytic function. Then, for
 any $m\in\mathbb{N}$, any sequence $x_1, \ldots, x_m$ of $m$ real numbers such that $0 < |x_1| < |x_2| < \ldots < |x_m|$ and for all but countably many $\gamma\in[0, 1]$, it holds that that $\mathbf{a}(\gamma x_1), \ldots, \mathbf{a}(\gamma x_m)$ are linearly independent over $\mathbb{Q}$. 
\end{lemma}

To apply this lemma, we have to make sure that among $x_1, \ldots, x_m$, no number is equal to $0$, and no two numbers have the same absolute value. This is because each $x_i$ is a linear combination of some $t$-layer feature vectors with \emph{positive integer} coefficients. This means that any equality of the form $x_i = 0$ or $x_i + x_j = 0$ would lead to a non-trivial linear combination with integer coefficient of some $t$-layer feature vectors, which is impossible due to their linear independence over $\mathbb{Q}$. It only remains to prove Lemma \ref{ind_lemma}.

\begin{proof}[Proof of Lemma \ref{ind_lemma}]
Let $\bad$ denote the set of $\gamma\in [0, 1]$ for which $\ba(\gamma x_1), \ldots, \ba(\gamma x_m)$ are linearly dependent over $\mathbb{Q}$. Our task is to show that the set $\bad$ is countable. By definition, for every $\gamma \in \bad$ there exists $(\lambda_1^\gamma, \ldots, \lambda_m^\gamma)\in\mathbb{Q}^m\setminus\{(0, 0, \ldots, 0)\}$,  such that:
    \[\lambda_1^\gamma \mathbf{a}(\gamma x_1) + \ldots + \lambda_m^\gamma \mathbf{a}(\gamma x_m) = 0.\]
    This defines a mapping from $\bad$ to a countable set $\mathbb{Q}^m\setminus\{(0, 0, \ldots, 0)\}$:
    \[\Psi\colon \gamma \mapsto (\lambda_1^\gamma, \ldots, \lambda_m^\gamma).\]
    To establish that the set $\bad$ is countable, it is enough to show that for any $(\lambda_1, \ldots, \lambda_m)\in \mathbb{Q}^m\setminus\{(0, 0, \ldots, 0)\}$ there exists just finitely many $\gamma\in[0, 1]$ such that
\[ \lambda_1 \ba(\gamma x_1) + \ldots + \lambda_m\ba(\gamma x_m) = 0.\]
Indeed, this would imply that $\Phi^{-1}((\lambda_1, \ldots, \lambda_m))$ is finite for every $(\lambda_1, \ldots, \lambda_m)\in \mathbb{Q}^m\setminus\{(0, 0, \ldots, 0)\}$, meaning that we  can write for the the set $\bad$:
\[\bad = \bigcup\limits_{\overline{\lambda}\in \mathbb{Q}^m\setminus\{(0, 0, \ldots, 0)\}}\Phi^{-1}(\overline{\lambda}),\]
which is a countable union of finite sets, which can only be countable.

To finish the proof, we take an arbitrary tuple of coefficients $(\lambda_1, \ldots, \lambda_m)\in \mathbb{Q}^m\setminus\{(0, 0, \ldots, 0)\}$ and show that the function:
\[g(\gamma) = \lambda_1 \mathbf{a}(\gamma x_1) + \ldots + \lambda_m \mathbf{a}(\gamma x_m)\]
has only finitely many roots in $[0, 1]$. Observe that the function $g$ is analytic. Assume for contradiction that it has infinitely many roots in $[0, 1]$. 
It is a well-known fact that an analytic function, having infinitely many roots in a bounded interval, must be the constant zero function. For the reader's convenience, we give a simple derivation of this fact from Proposition \ref{der}. Any bounded infinite set of real numbers has an accumulation point, that is, a point $\gamma_0$ such that every neighborhood of $\gamma_0$ has a point from the set other than $\gamma_0$. We apply this fact to the set of roots of $g$ in $[0, 1]$. We observe that all derivatives of $g$ at $\gamma_0$ must be equal to $0$, implying that $g$ is the constant zero function by Proposition \ref{der}. Indeed, otherwise we can write $g(\gamma) = \frac{g^{(k)}(\gamma_0)}{k!}(\gamma - \gamma_0)^k + o((\gamma - \gamma_0)^k)$ as $\gamma\to\gamma_0$ for the smallest $k$ with $g^{(k)}(\gamma_0)\neq 0$, which would imply that $g(\gamma)$ is different from $0$ for all $\gamma\neq \gamma_0$ sufficiently close to $\gamma_0$.

Since $g$ is the constant zero function, we have:
   \[g^{(k)}(0) = (\lambda_1 x_1^k + \ldots + \lambda_m x_m^k) \mathbf{a}^{(k)}(0)\]
   for every $k\ge 0$. In turn,  since $\mathbf{a}$ is a non-polynomial analytic function, we have $\mathbf{a}^{(k)}(0)\neq 0$ for infinitely many\footnote{Indeed, if all derivatives of $\mathbf{a}$ of sufficiently large order are equal to $0$ at $0$, by Proposition \ref{der} we have that some derivative of $\ba$ is the constant zero function, meaning that $\ba$ itself is a polynomial.} $k$, meaning that
 \[\lambda_1 x_1^k + \ldots + \lambda_m x_m^k = 0\] 
 for infinitely many $k$.

   Take the largest $i\in\{1, \ldots, m\}$ such that $\lambda_i \neq 0$ (it exists because not all $\lambda_i$ are equal to $0$).
  Then we have:
   \begin{align*}
&\lim\limits_{k\to\infty}\frac{\lambda_1 x_1^k + \ldots + \lambda_m x_m^k}{x_i^k}\\ &= \lim\limits_{k\to\infty} \lambda_1\left(\frac{x_1}{x_i}\right)^k + \ldots + \lambda_i\left(\frac{x_i}{x_i}\right)^k .
   \end{align*}
   Since $0 < |x_1| < |x_2| < \ldots < |x_m|$, this limit is equal to $\lambda_i \neq 0$. However, the expression under the limit
    is equal to $0$ for infinitely many $k$, a contradiction.
\end{proof}

\section{Experiments}

To precisely describe our experiments, let us start by establishing some terminology. We say that two labelings $l_1$ and $l_2$ of the nodes of a given graph $G$ are \emph{equivalent} if for any two nodes $u,v$ in $G$, it holds that $l_1(u)=l_1(v)\iff l_2(u)=l_2(v)$. For a given graph $G$, we say that the WL algorithm  \emph{converges} in $T$ steps if $T$ is the first time for which the labelings at $T$ and $T+1$ are equivalent. It is not hard to see that in this case the labeling of $G$ at iteration $T$ is equivalent to the labeling at $T+t$ for all $t>0$. Given a graph $G$ on which WL converges in $T$ iterations, we say that a given MPNN $\mathcal{M}$ \emph{perfectly simulates} WL on $G$ if the labeling assigned to $G$ by $\mathcal{M}$ after $T$ iterations is equivalent to the one assigned by WL. 

For simplicity, we performed the experimental analysis for non-labeled graphs only. In this case, the architecture of our construction from \Cref{sec:proof} can be simplified to the following specifications for a MPNN $\mathcal{M}_\gamma$:

\begin{itemize}
    \item features initialization: \[f^{(0)}(v)=1, \quad v\in V\]
    \item features update: \[f^{(t+1)}(v) = \mathbf{a}\left(\gamma f^{(t)}(v) + \gamma \sum\limits_{u\in N(v)}f^{(t)}(u)\right)\] where $\gamma \in (0,1)$ is the same for all the iterations.  
\end{itemize}

Our theoretical result for this case predicts that for all $\gamma \in (0,1)$, except possibly for a countable number, $\mathcal{M}_\gamma$ perfectly simulates WL on every graph $G$. We notice that the validity of this theoretical result, however, depends on the assumption that $\mathcal{M}_\gamma$ is \emph{continuous}, in the sense that operations on real numbers are performed with arbitrarily high precision.

\paragraph{Uniform perfect simulation} To demonstrate the empirical validity of our theoretical statement, we generated 50 values of $\gamma$ chosen uniformly at random, and ran $\mathcal{M}_\gamma$ with sigmoid activation function on two different collections of random graphs: i) a collection of 300 sparse Erdos-Renyi random graphs, and a collection of 300 sparse scale-free random graphs (generated with the Barab\'asi-Albert algorithm). In both cases, the graph sizes range from 50 to 250 nodes. As illustrated in Figure \ref{fig:sacos}, we obtained that for nearly all values of $\gamma$ (45 in the case of Barabasi, and 47 for Erdos-Rendyi), $\mathcal{M}_\gamma$ perfectly simulated WL on all graphs. 

\begin{figure}[!h]
    \centering
    \includegraphics[width=7cm]{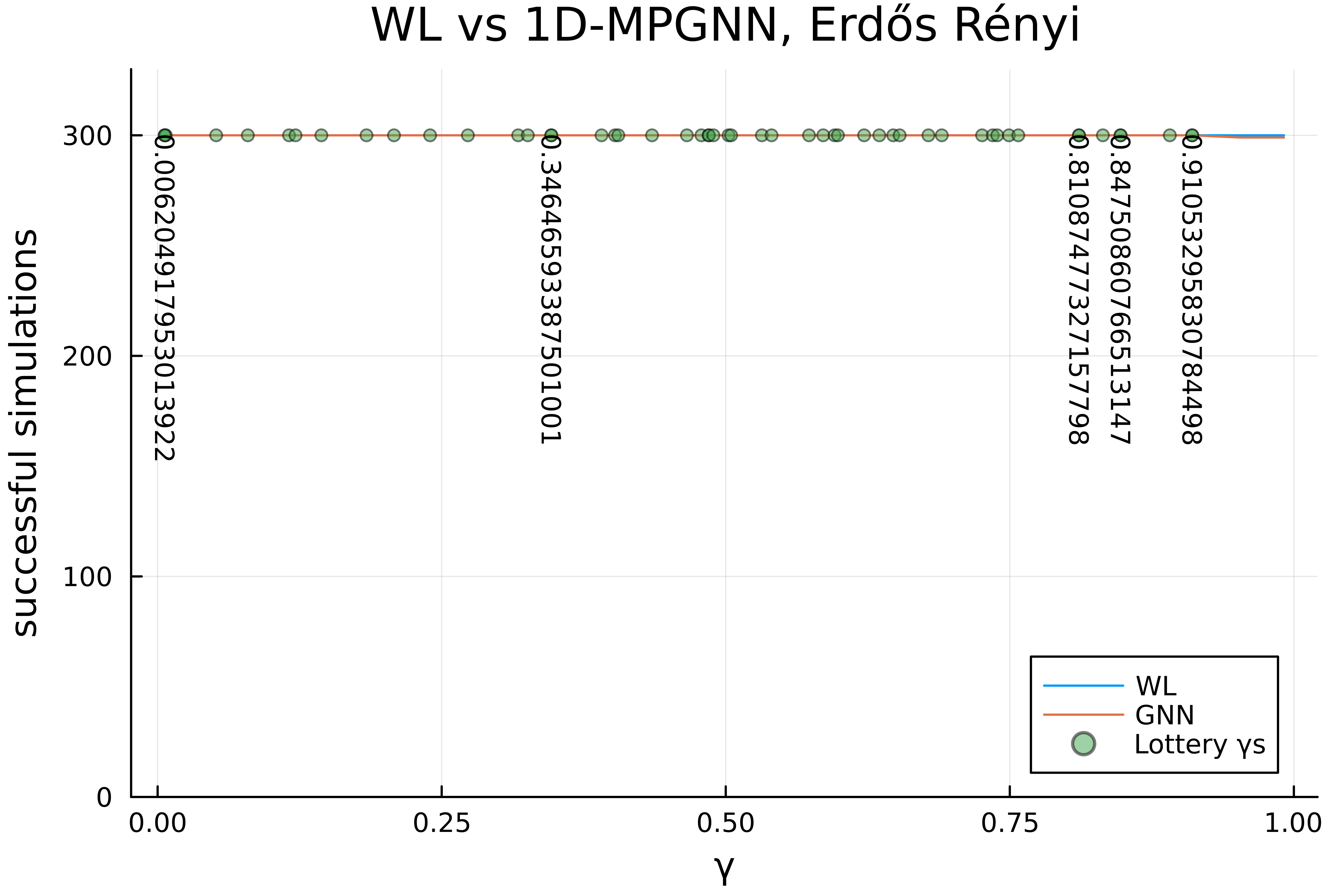}
    \includegraphics[width=7cm]{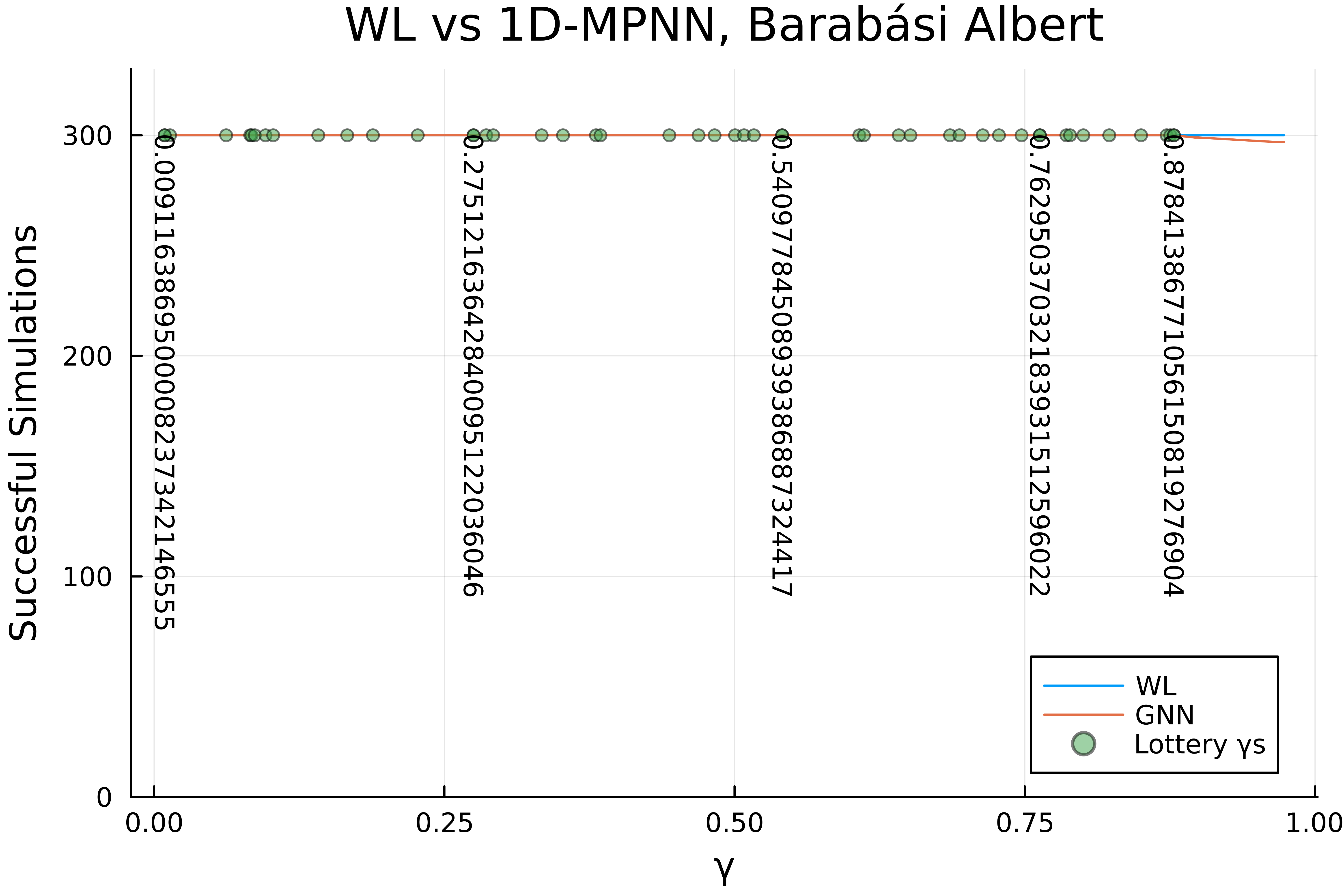}
    \caption{Number of perfect simulations of $\mathcal{M}_\gamma$ among the 300 random graphs, for 50 different values of $\gamma$ chosen uniformly at random. Lottery $\gamma$s are those that achieved perfect simulation on all the 300 graphs -- 45 in the case of Erdos-Renyi graphs, and 44 for scale-free graphs. For each case, the value of a few $\gamma$s is also shown.}
    \label{fig:sacos}
\end{figure}

\paragraph{Dependence on precision bits} To study how the number of precision bits required to achieve perfect simulation is affected by the size of the input graphs, we performed the following experiment: we generated random Erdos-Renyi graphs of increasing sizes, from $n=50$ to $n=3000$, together with 20 values for $\gamma$ chosen uniformly at random between $0$ and $0.5$\footnote{We experimentally observed that smaller values for $\gamma$ tend to work better, which explains the chosen range for this experiment.}. For each graph and each value of $\gamma$, we computed the minimum number of bits required for perfect simulation. Figure \ref{fig:bits} shows, for each $n$, the average over $\gamma$ of this minimum precision, together with the standard deviations. The obtained curve is approximately $O(\log(n))$.

\begin{figure}[!h]
    \centering
    \includegraphics[width=7cm]{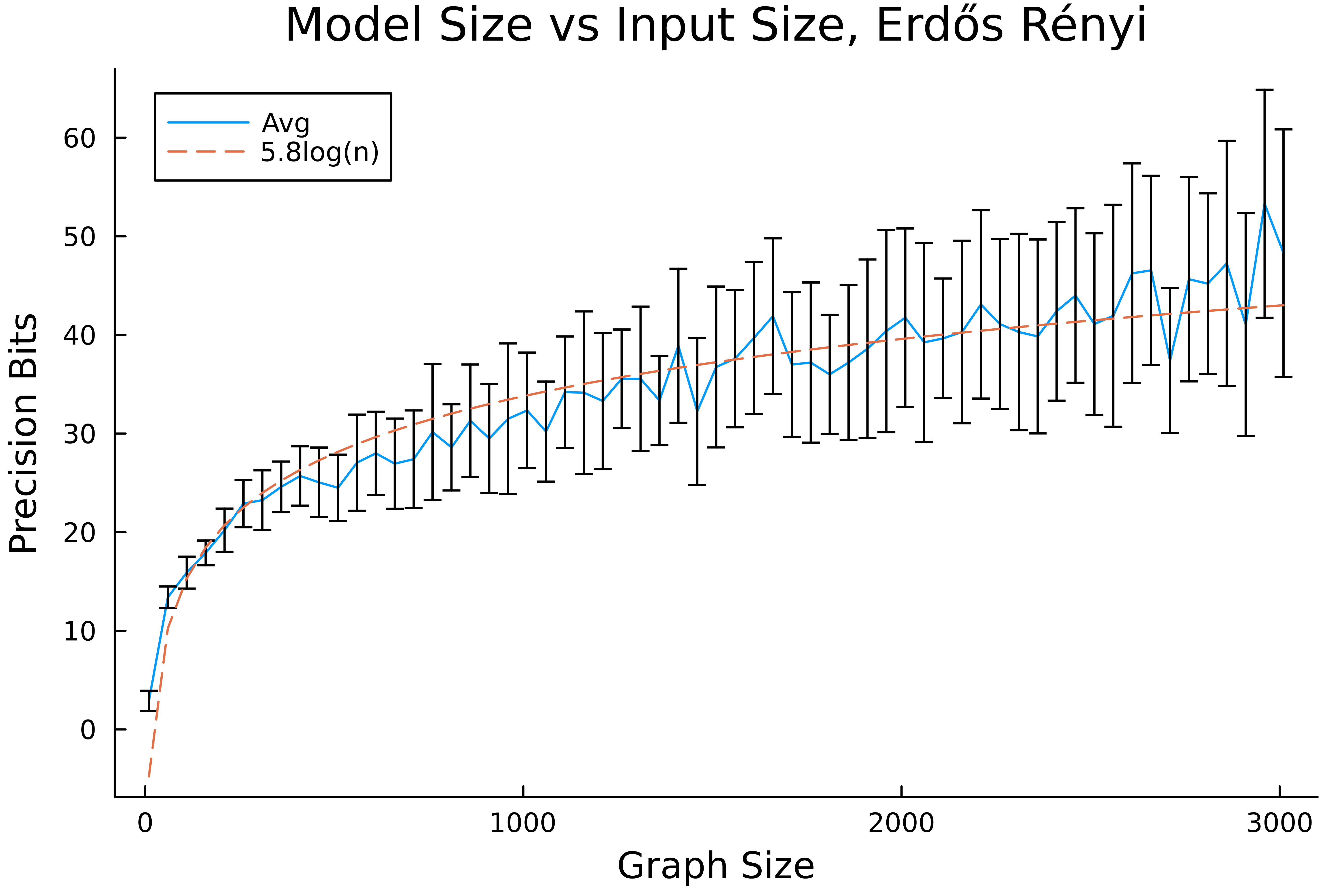} 
    
    \includegraphics[width=7cm]{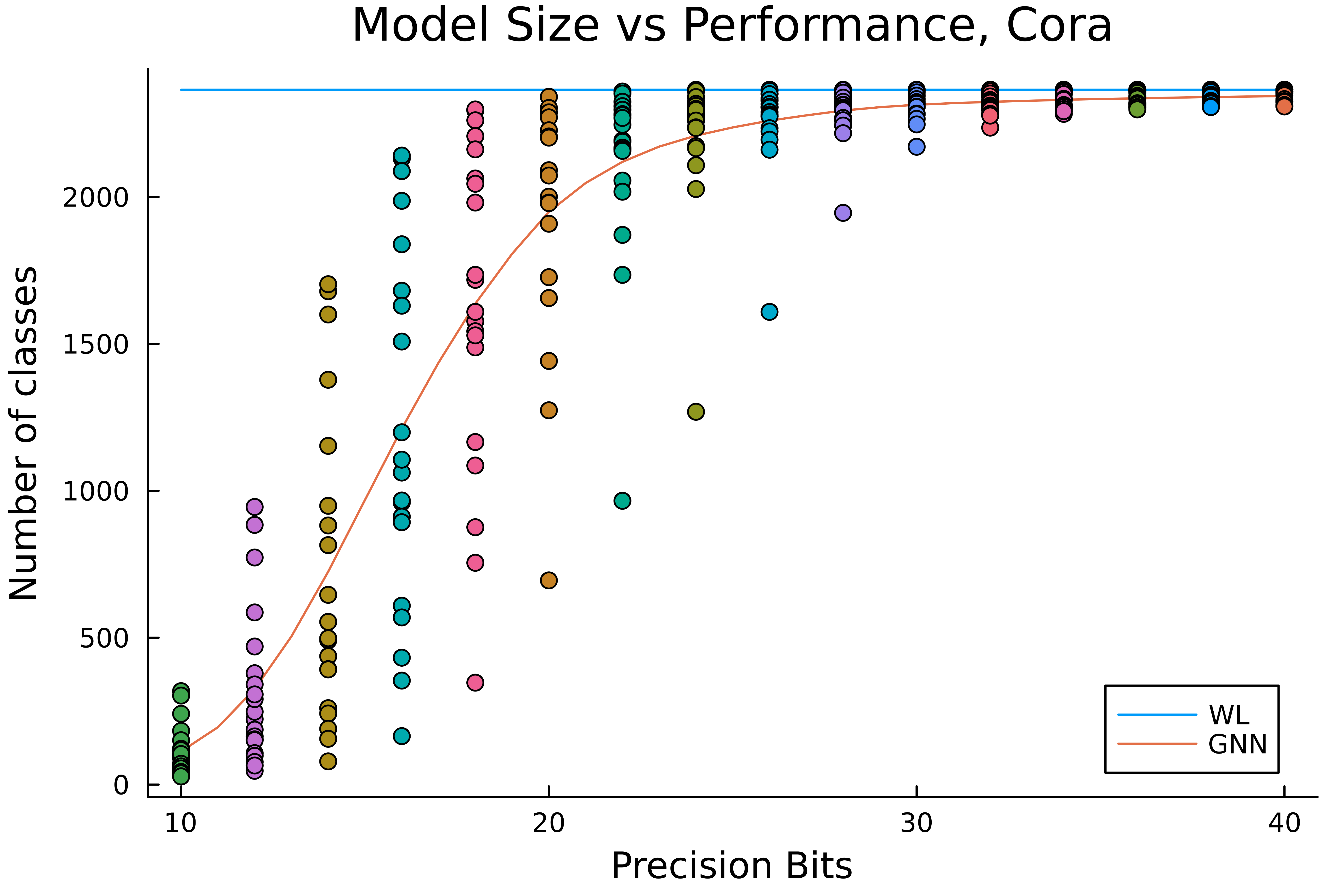}
    
    \includegraphics[width=7cm]{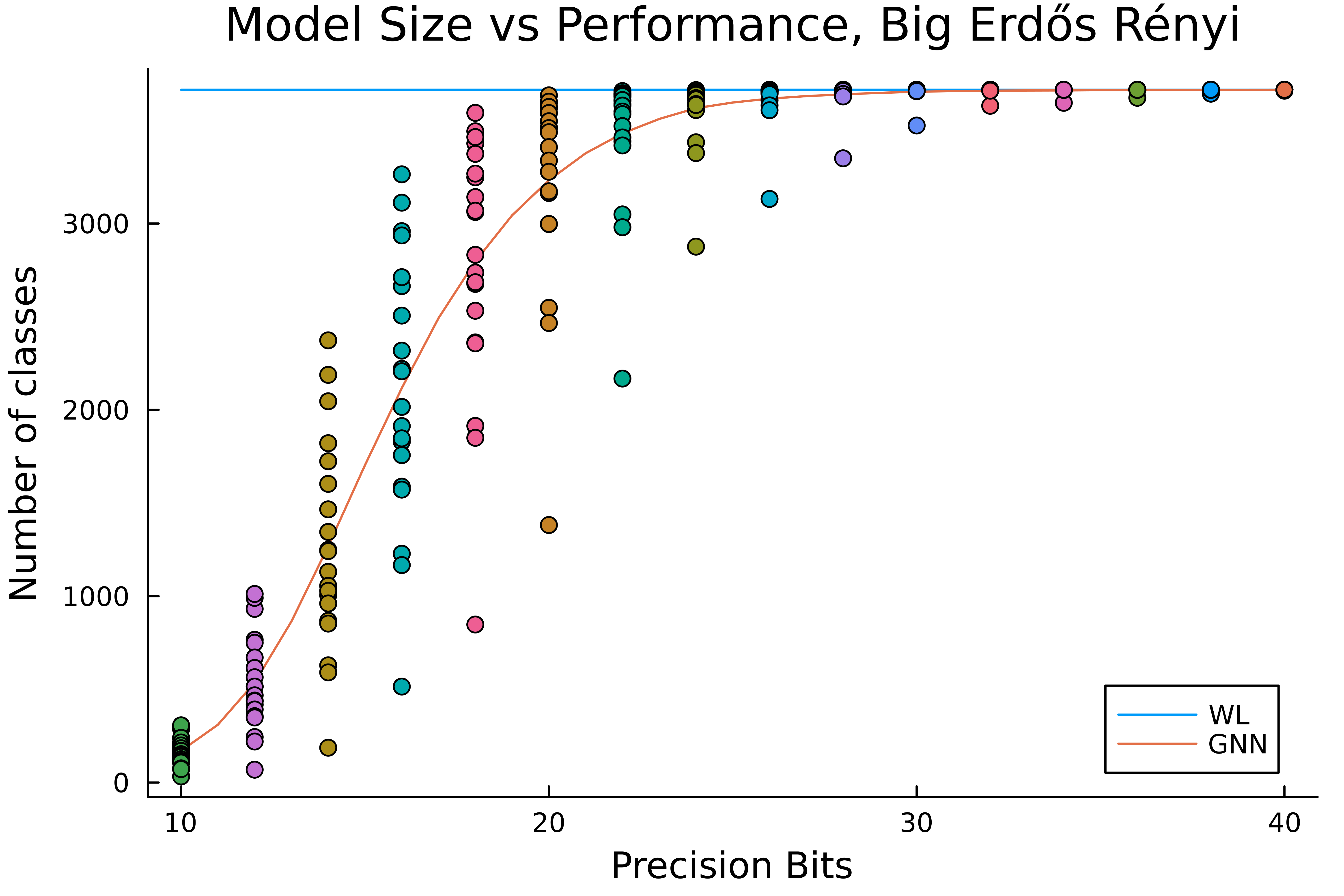}
    \caption{The figure on top shows the number of precision bits required for perfect simulation as a function of the size of the graph. The other figures display the quality of the simulation as a function of the number of precision bits. For each precision and each value of $\gamma$, the number of classes in the final labelling output by $\mathcal{M}_
    \gamma$ are potted. The true number of classes in the case of CORA is 2365, whereas for the big (5000 nodes) Erdos-Renyi graph is 3729.}
\label{fig:bits}
\end{figure}

Finally, we analyzed the real-world graph from the database CORA\footnote{\url{https://graphsandnetworks.com/the-cora-dataset/}} as well as a big Erdos-Renyi random graph with $5000$ nodes and measured how the performance of the one-dimensional MPNNs improves as a function of the number of precision bits. Figure \ref{fig:bits} shows,  for 20 randomly chosen values of $\gamma$, the number of classes obtained by $\mathcal{M}_\gamma$ as a function of the number of precision bits. We have made available the code of all our experiments\footnote{\url{https://anonymous.4open.science/r/Single-channel-GNN-B05F/}}.

%

\section{Higher-order MPNN and WL tests}
We briefly discuss how to extend our results to the higher-order version of the WL test and corresponding higher-order MPNN. We are using a so-called \emph{non-folklore}  version of the WL test (NWL), not to be confused with the folklore one, see~\cite{huang2021short}.

Fix $k\ge 2$. We will work over graphs with labelings of their $k$-tuples of nodes, formally, with triples of the form $G = \langle V, E, \well\rangle$, where $\well\colon V^k \to L$. If we are given just an unlabeled graph $G = \langle V, E \rangle$, a standard way to define $\well$ is via \emph{isomorphism-type matrices}:
\[\well(\mathbf{v}) = A_\mathbf{v}, \qquad \mathbf{v} = (v_1, \ldots, v_k)\in V^k,\]
where $A_\mathbf{v}$ is a $k\times k$ matrix, defined by:
\[(A_\mathbf{v})_{ij} = \begin{cases}
    \mathrm{EDGE} & \{v_i, v_j\} \in E, \\
    \mathrm{NON-EDGE} & \{v_i, v_j\} \notin E \text{ and } v_i\neq v_j, \\
    \mathrm{EQUAL} &  v_i = v_j.
\end{cases}\]

The $k$-order NWL test, given a graph $G = \langle V, E, \well\rangle$, assigns it an infinite sequence of labelings $\{\phi^{(t)}_G\}_{t = 0}^\infty$ of $k$-tuples of nodes of $G$ (formally, these are functions whose domain is $V^k$). The initial labeling in this sequence coincides with $\well$:
\[\phi^{(0)}_G(\mathbf{v}) = \well(\mathbf{v}), \qquad \mathbf{v}\in V^k.\]
To define the update rule, we have to introduce some ``neighbor'' notation for $k$-tuples. For a tuple $\mathbf{v} = (v_1, \ldots, v_k)\in V^k$ and for $i\in\{1, 2,\ldots, k\}$, we define the set:
\[N_i(\mathbf{v}) = \left\{(v_1, \ldots, v_{i - 1}, u, v_{i + 1}, \ldots, v_k) \mid u\in V\right\}\]
of tuples that coincide with $\mathbf{v}$ in all coordinates except, possibly, $i$ (if one represents $V^k$ as a hypercube, the set $N_i(\mathbf{v})$ will consist of $\mathbf{v}$ and its neighbors by edges, parallel to the $i$th dimension). The update is defined as follows:
\[\phi^{(t+1)}_G(\mathbf{v}) = \left(\phi^{(t)}_G(\mathbf{v}), \bleftm \phi^{(t)}_G(\mathbf{u}) \mid \mathbf{u}\in N_1(\mathbf{v})\brightm, \ldots, \bleftm \phi^{(t)}_G(\mathbf{u}) \mid \mathbf{u}\in N_k(\mathbf{v})\brightm \right)\]
Given two graphs $G_1$ and $G_2$ with labeled $k$-tuples, the $k$-order NWL test distinguishes them if and only if there exists $t\ge 0$ such that
\[\bleftm \phi_{G_1}^{(t)}(\mathbf{v})\mid \mathbf{v}\in V^k\brightm \neq \bleftm \phi_{G_2}^{(t)}(\mathbf{v})\mid \mathbf{v}\in V^k\brightm\]

It is well-known that the 2-order NWL test is equivalent to the standard WL test (defined in the main body of the paper), and in general, $k$-order NWL is equivalent to the $(k - 1)$-order folklore WL test~\cite{huang2021short}.

\medskip

Next, we introduce $k$-order non-folklore MPNNs. 
A  $T$-layer $k$-order non-folklore MPNN $\mathcal{N}$ with $d$-dimensional feature vectors and activation function $a\colon \mathbb{R}\to\mathbb{R}$ is specified by:
\begin{itemize}
    \item a label-encoding function $\mathbf{e}\colon L\to\mathbb{R}^d$;
    \item a $T$-length sequence of $(k+1)$-tuples of $d\times d$ real matrices:
\[\left\{(W_0^{(t)}, W_1^{(t)}, \ldots, W_k^{(t)})\right\}_{t = 1}^T.\]    
\end{itemize}

Given a  graph $G = \langle V, E, \well\rangle$ with the labeling of $k$-tuples of nodes, the $k$-order non-folklore MPNN $\mathcal{N}$ functions on it as follows. The initial feature vectors are computed as the composition of the label encoding and initial labeling:
\[f^{(0)}(\mathbf{v}) = \mathbf{e}(\well(\mathbf{v}))\in\mathbb{R}^d, \qquad \mathbf{v}\in V^k. \]
Then for $t = 0, \ldots, T-1$, the updates  of the feature vectors are performed as follows:
\begin{align}
\label{k_gnn_update}
    f^{(t+1)}(\mathbf{v}) &= \mathbf{a}\left(W_0^{(t+1)} f^{(t)}(\mathbf{v}) + W_1^{(t+1)}\sum\limits_{\mathbf{u}\in N_1(\mathbf{v})} f^{(t)}(\mathbf{u})+\ldots+ W_k^{(t+1)}\sum\limits_{\mathbf{u}\in N_k(\mathbf{v})} f^{(t)}(\mathbf{u})\right),\\
    f^{(t+1)}(\mathbf{v})&\in\mathbb{R}^d, \mathbf{v}\in V^k.
\end{align}
The function $\mathbf{a}$ is applied component-wise.

The output of $\mathcal{N}$ on the graph $G$ is defined as the sum of feature vectors of all the $k$-tuples after $T$ iterations:
\[\mathcal{N}(G) = \sum\limits_{\mathbf{v}\in V^k} f^{(T)}(v).\]
As before,  two graphs $G_1, G_2$ are distinguished by $\mathcal{N}$ if $\mathcal{N}(G_1) \neq \mathcal{N}(G_2)$.

\begin{theorem}
\label{thm_higher}
For any non-polynomial analytic function $\mathbf{a}\colon\mathbb{R}\to\mathbb{R}$ the following holds.
Two graphs can be distinguished by the $k$-order NWL test if and only if they can be distinguished by some $k$-order non-folklore MPNN with 1-dimensional feature vectors and activation function $\mathbf{a}$.
\end{theorem}
\begin{proof}
The proof of Theorem \ref{thm_main} works almost verbatim.  As before, everything reduces to a construction, for a given graph $G = \langle V, E, \well\rangle$, and for any number of iterations $T$, of a $k$-order MPNN $\mathcal{N}$ such that, for any $t = 0, 1, \ldots, T$, we first have one-to-one correspondence between feature vectors and WL labels:
\begin{equation}
\label{equiv_lab}
f^{(t)}(\mathbf{v}) = f^{(t)}(\mathbf{u}) \iff \phi^{(t)}(\mathbf{v}) = \phi^{(t)}(\mathbf{u}), \qquad \mathbf{v}, \mathbf{u}\in V^k, 
\end{equation}
and second, the set of feature ``vectors'' (scalars, in our case) $\{f^{(t)}(\mathbf{v})\mid \mathbf{v}\in V^k\}$ is linearly independent over $\mathbb{Q}$.

The proof is, as before, by induction on $t$, and to go from $t$ to $t +1$, we construct a layer of the form:
\begin{align*}
f^{(t+1)}(\mathbf{v}) &= \mathbf{a}\left(\gamma \cdot x_\mathbf{v}\right), \\
x_\mathbf{v} &=  f^{(t)}(\mathbf{v}) + (n+1)\sum\limits_{\mathbf{u}\in N_1(\mathbf{v})} f^{(t)}(\mathbf{u})+\ldots+ (n+1)^k\sum\limits_{\mathbf{u}\in N_k(\mathbf{v})} f^{(t)}(\mathbf{u})
\end{align*}
where $n$ is the number of nodes of $G$ and $\gamma\in [0, 1]$ is a parameter to be chosen later (this is a special case of \eqref{k_gnn_update}). We need to observe that $x_\mathbf{v}$ uniquely determines $\phi^{(t+1)}(\mathbf{v})$ for $\mathbf{v}\in V^k$. Indeed,  to determine
\[\phi^{(t+1)}(\mathbf{v}) = \left(\phi^{(t)}(\mathbf{v}), \bleftm \phi^{(t)}(\mathbf{u}) \mid \mathbf{u}\in N_1(\mathbf{v})\brightm, \ldots, \bleftm \phi^{(t)}(\mathbf{u}) \mid \mathbf{u}\in N_k(\mathbf{v})\brightm \right),\]
it is enough, by \eqref{equiv_lab}, to determine:
\[\left(f^{(t)}(\mathbf{v}), \bleftm f^{(t)}(\mathbf{u}) \mid \mathbf{u}\in N_1(\mathbf{v})\brightm, \ldots, \bleftm f^{(t)}(\mathbf{u}) \mid \mathbf{u}\in N_k(\mathbf{v})\brightm \right).\]
For this, in turn, it is enough to determine, for every $f\in\{f^{(t)}(\mathbf{v})\mid \mathbf{v}\in V^k\}$, whether it is equal to $f^{(t)}(v)$, and how many times it appears in each of the sums:
\[\sum\limits_{\mathbf{u}\in N_1(\mathbf{v})} f^{(t)}, \ldots, \sum\limits_{\mathbf{u}\in N_k(\mathbf{v})} f^{(t)}.\]
Since the set $\{f^{(t)}(\mathbf{v})\mid \mathbf{v}\in V^k\}$ is linearly independent over $\mathbb{Q}$, knowing $x_\mathbf{v}$, one can determine the coefficient before $f$ in the sum, defining $x_\mathbf{v}$, and this coefficient will be of the form $c_f = a_0 + (n + 1) a_1 + \ldots + (n+1)^k a_k$, where $a_0 = \mathbb{I}\{f = f^{(t)}(v)\}$, and $a_i$ is the number of occurrences of $f$ in the sum $\sum_{\mathbf{u}\in N_i(\mathbf{v})} f^{(t)}$, for $i = 1, \ldots, k$. Each of the sum consists of $n$ terms, meaning that $0\le a_0, a_1,\ldots,a_k\le n$, and this implies that $a_0, \ldots, a_k$ are the digits of $c_f$ in the $(n+1)$-base expansion.

Thus, the numbers $x_{\mathbf{v}}, \mathbf{v}\in V^k$ are in the one-to-one correspondence with the WL labels after $t+1$ iterations. It remains to make them linearly independent over $\mathbb{Q}$ again, by applying the transformation:
\[x_{\mathbf{v}} \mapsto \mathbf{a}(\gamma \cdot x_{\mathbf{v}})\]
for a suitable choice of $\gamma$, which again exists by the \Cref{ind_lemma}. No two numbers $x_{\mathbf{v}}, x_{\mathbf{u}}$ can sum up to $0$ as it would give a non-trivial linear combinations with rational coefficients of feature vectors after $t$ iterations. Thus, the numbers  $x_{\mathbf{v}}, \mathbf{v}\in V^k$ satisfy the conditions of Lemma \ref{ind_lemma}.
\end{proof}

\section{Conclusion and future work}

In this paper, we have demonstrated that, as far as expressive power is concerned, the dimensionality of feature vectors is not a restricting factor in the design of MPNNs. Besides providing some theoretical justification of the behavior already observed in practice for these architectures, there are some consequences related to generalization that may be relevant for practical considerations. As shown in \cite{morris2023wl}, low-dimensional architectures exhibit significantly better generalization performance, as compared to higher-dimensional ones. According to the experiments they report, the difference between training and testing accuracies was systematically below 5\% across all data sets for $d=4$, whereas for $d=1024$ this difference can become more than 45\% for some data sets. The accuracy itself for low-dimensional architectures, however, was observed to vary significantly along different data sets, ranging from 30\% of accuracy for some data sets to more than 90\% for others. 

What is the reason for this discrepancy in accuracy across different data sets?

A natural hypothesis is that such low-dimensional architectures have limited expressive power, and thus there exist data sets that are simply impossible to fit for them. Our results, however, suggest that this may not necessarily be the case. The explanation, therefore, could rather be related to the training process itself (relative to the data), than to the intrinsic limitations of the architecture. How does this trade-off work, and how to find the optimal dimensionality for different practical applications, are interesting topics for future research.

\end{document}